\documentclass{article}

\usepackage{microtype}
\usepackage{graphicx}
\usepackage{subfigure}
\usepackage{booktabs} 

\usepackage[utf8]{inputenc}
\usepackage{xcolor}
\usepackage[boxed]{algorithm}
\usepackage{algorithmic}
\usepackage{amsmath}
\usepackage{amsthm}
\usepackage{amssymb}
\usepackage{comment}
\usepackage{authblk}
\usepackage{placeins}
\usepackage{caption}
\usepackage{multirow}
\usepackage{fullpage}
\usepackage{xspace}

\usepackage{mathrsfs}

\usepackage{hyperref}
\usepackage[capitalize]{cleveref}
\usepackage[nolist,nohyperlinks]{acronym}
\usepackage{dsfont}

\usepackage{enumitem}

\usepackage{mathtools}

\usepackage{array}
\usepackage{tabularx}

\usepackage{natbib}
\setcitestyle{round,sectionbib}

\def\cX{{\mathcal{X}}}
\def\cY{{\mathcal{Y}}}

\def\real{{\mathbb{R}}}

\newcommand\Itrain{\relax\ifmmode I_{\text{train}}\else $I_{\text{train}}$\xspace\fi}
\newcommand\Ical{\relax\ifmmode I_{\text{cal}}\else $I_{\text{cal}}$\xspace\fi}
\newcommand\Itest{\relax\ifmmode I_{\text{test}}\else $I_{\text{test}}$\xspace\fi}

\newcommand\Bcal{\relax\ifmmode B_{\text{cal}}\else $B_{\text{cal}}$\xspace\fi}
\newcommand\Btest{\relax\ifmmode B_{\text{test}}\else $B_{\text{test}}$\xspace\fi}

\newcommand\E{\mathbb{E}}
\newcommand\Prob{\mathbb{P}}

\def\one{{\mathbb{I}}}

\renewcommand{\P}{\mathbb{P}}
\newcommand{\pr}[1]{\left( #1 \right)}

\newcommand{\cbr}[1]{\left\{ #1 \right\}}

\newcommand*\diff{\mathop{}\!\mathrm{d}}
\newcommand{\diam}{\mathrm{diam}}
\newcommand{\cF}{\mathcal{F}}
\newcommand{\cZ}{\mathcal{Z}}
\newcommand{\cD}{\mathcal{D}}

\DeclareMathOperator{\argmin}{arg min}

\begin{acronym}
\acro{CRC}{Conformal Risk Control}
\acro{LTT}{Learn-then-Test}
\acro{LLM}{Large language model}
\acro{UCB}{upper confidence bound}
\acro{RCPS}{Risk-Controlling Prediction Sets}
\end{acronym}

\acrodefplural{LLM}[LLMs]{Large language models}

\theoremstyle{plain}
\newtheorem{theorem}{Theorem}[section]
\newtheorem{proposition}[theorem]{Proposition}
\newtheorem{lemma}[theorem]{Lemma}

\theoremstyle{definition}

\theoremstyle{remark}

\usepackage[
textsize=tiny]{todonotes}

\makeatletter
\newcommand{\printfnsymbol}[1]{%
  \textsuperscript{\@fnsymbol{#1}}%
}
\makeatother

\begin{document}

\title{\LARGE \bf
Mitigating LLM Hallucinations via Conformal Abstention
}
\author{Yasin Abbasi-Yadkori\thanks{Equal contribution}, Ilja Kuzborskij\printfnsymbol{1}, David Stutz, Andr\'{a}s Gy\"{o}rgy, Adam Fisch, Arnaud Doucet, Iuliya Beloshapka, Wei-Hung Weng, Yao-Yuan Yang, Csaba Szepesv\'{a}ri, Ali Taylan Cemgil, Nenad Tomasev}

\maketitle

\begin{abstract}
We develop a principled procedure for determining when a large language model (LLM) should abstain from responding (e.g., by saying ``I don't know'') in a general domain, instead of resorting to possibly ``hallucinating'' a non-sensical or incorrect answer. Building on earlier approaches that use self-consistency as a more reliable measure of model confidence, we propose using the LLM itself to self-evaluate the similarity between each of its sampled responses for a given query. We then further leverage conformal prediction techniques to develop an abstention procedure that benefits from rigorous theoretical guarantees on the hallucination rate (error rate). Experimentally, our resulting conformal abstention method reliably bounds the hallucination rate on various closed-book, open-domain generative question answering datasets, while also maintaining a significantly less conservative abstention rate on a dataset with long responses (Temporal Sequences) compared to baselines using log-probability scores to quantify uncertainty, while achieveing comparable performance on a dataset with short answers (TriviaQA). To evaluate the experiments automatically, one needs to determine if two responses are equivalent given a question. Following standard practice, we use a thresholded similarity function to determine if two responses match, but also provide a method for calibrating the threshold based on conformal prediction, with theoretical guarantees on the accuracy of the match prediction, which might be of independent interest. 
\looseness=-1
\end{abstract}

\section{Introduction}
\label{sec:introduction}
\acp{LLM} are excellent at next word prediction. At the same time, however, they are also prone to \emph{hallucination}---that is, confidently generate responses that may look plausible on the surface, but that are actually incorrect or even nonsensical~\cite{ji2023survey, maynez2020faithfulness}. 
Unfortunately, hallucinations are difficult to detect, especially when users are not able to easily verify the factuality of an 
LLM's responses by themselves. In  generation tasks in particular, it can be  challenging to discriminate between hallucinations that present false facts, and any of the many other viable ways of expressing correct information. Therefore, hallucinations can be extremely detrimental towards achieving trustworthy and reliable LLM performance, and hence avoiding or even detecting hallucinations has become one of the most important research topics in LLM research.

In this work, we develop a principled  abstention policy that mitigates LLM hallucination by simply choosing to either produce a single response from the model that is likely to be hallucination-free,  or otherwise abstain from producing a response altogether (e.g., by saying ``I don't know''). The quality of such a policy can be measured by two quantities: the expected proportion of time the method chooses to abstain, and the expected proportion of unfiltered hallucinations in the responses; we will henceforth refer to these as the \emph{abstention rate} and the hallucination \emph{risk}, respectively. 

While directly considering the (log-)probabilities of the response sequence generated by an LLM might be tempting, these probabilities heavily depend on the length of the output sequence, and the likelihood of an answer becomes non-indicative of its correctness as the sequence length grows \citep{MLG2023}.
Therefore, a large body of prior work has attempted to detect hallucinations through either confidence estimation \citep{CZGEDE2023, MLG2023,KuhnARXIV2023,wang2023selfconsistency}
or more involved inference time procedures. 
A consistent observation that has been reported in prior work is that \emph{uncertainty} of the LLM responses, or equivalently, the level of \emph{agreement} between a batch of sampled responses, tends to be a reasonable proxy for detecting hallucinations, although it clearly cannot detect situations where the LLM is completely sure about an incorrect answer.
This approach comes with two immediate challenges: how we can decide if two responses agree for a given question, and what level of disagreement indicates hallucination.

In this paper we address both of these questions, by (i) developing well-engineered prompts to use the LLM for evaluating the similarity of two of its responses for a given query; and (ii) using theoretically well-founded methods to determine the level of agreement in evaluation responses, below which the LLM is likely hallucinating. 
A crucial property of (i) is that the \emph{self}-evaluation prompt depends on the query itself, making it explicit that similarity of two responses depends on the question.
For (ii), we leverage the \emph{conformal prediction} and related risk control techniques \citep{vovk2005algorithmic, bates2021distribution, AngelopoulosARXIV2021c, AngelopoulosARXIV2022}, by assuming access to a small holdout calibration set of prompt-response pairs.
These techniques allow us to \emph{calibrate} the detection/abstention policy so that it satisfies a pre-specified, distribution-free, statistical upper bound on the hallucination risk while minimizing the abstention rate. 
Our method is lightweight as it is only based on prompting and does not require to update the LLM itself, such as by fine-tuning.

We evaluate our method on a variety of closed-book open-domain question answering tasks (using a Gemini Pro model, \citealt{geminiteam2023gemini}). 
In particular, as also observed in parallel work \citep{KuhnARXIV2023,MLG2023}, we find that an instruction-tuned LLM can effectively and efficiently be used not only to generate candidate responses, but also to self-evaluate the coherence among responses; we then use the latter either to select a final response or to choose to abstain. We find that abstention with self-evaluation outperforms log-probability baselines used in the literature \citep{QuachARXIV2023, azaria2023internal}. 

To evaluate the experiments automatically, one needs to determine if two responses are equivalent given a question. A standard way to do this is to  use a thresholded similarity function to determine if two answers match~\citep{QuachARXIV2023}. To select the right threshold, we provide a calibration method, also based on conformal prediction, which comes with theoretical guarantees on the accuracy of the match prediction, and applicable for small calibration datasets (which need to be labelled manually). To our knowledge, this is the first such method presented in the literature, and hence it might be of independent interest.

\section{Problem definition}

We now give a formal definition of the problem we consider and summarize our approach. Let $\cX$ be a space of input prompts and $\cY$ be a space of output responses. Let $m:\cX\times\cY\times\cY\rightarrow \{0,1\}$ be the binary ground-truth match function, so that $m(X; Y',Y)=1$ indicates that response $Y'\in \cY$ matches the response $Y\in\cY$ for a given query $X\in\cX$, and $m(X; Y',Y)=0$ denotes that it does not. That is, given a ground truth response $Y$ to $X$, $m(X; Y',Y)$ is the indicator function whether $Y'$ is semantically equivalent to $Y$ \emph{given} $X$. The conditioning on $X$ makes our model very flexible: While the simplest way to define $m$ could be to check if $Y$ and $Y'$ mean the same thing, our setting can accommodate much broader and more useful definitions, the most appealing of which is whether $Y$ and $Y'$ are equally correct responses to $X$. For example, for the prompt $X=$\emph{``Tell me a European capital.''}, $Y=$\emph{``London''} is as good as $Y'=$\emph{``Paris''}, allowing our method to be applicable for questions with multiple different correct responses, as long as a good match function $m$ can be devised.

Given a classifier (i.e., a possibly random map) $f:\cX\rightarrow\cY$, its loss on the prompt-response pair $(X,Y)$ is defined as $1-m(X;f(X),Y)$.
Our goal is to obtain, given a classifier $f$, a selective classification scheme which can abstain from prediction (answering a prompt) when $f$ would make a mistake. To this end, we define an \emph{abstention} function, which can decide whether the classifier should be applied to a given input prompt $X$. We consider score-base abstention functions, that is, 
for a given parameter $\lambda \in \Lambda$ (where $\Lambda \subset \real$ is a a parameter space), a query $X \in \cX$, and a score function $g:\cX\rightarrow \real$ indicating the model's confidence in classifying the input, the abstention policy $a:\Lambda\times\cX\rightarrow \{0,1\}$ is defined as
\[
a_{\lambda}(X) = 
\left\{ \begin{array}{rcl}
1 & \mbox{if} & g(X) < \lambda \\ 
0 & \mbox{if} & g(X) \ge \lambda
\end{array}\right.
\]
where $a_{\lambda}(X)=1$ means that the predictor should abstain.
Given a query $X$, the score might be a random variable, and therefore $a$, similarly to $f$, might also be random.   
Together the pair $(a_\lambda, f)$ define a
\emph{selective classifier}.

Let $\ell:\cX\times\cY\times\Lambda\rightarrow\real$ be a loss function so that $\ell(X,Y;\lambda)$ is the loss of selective classifier $(a_\lambda, f)$ given a query-response pair $(X,Y)$. $\ell$ penalizes a policy when it does not abstain and its response does not match the label: 
\begin{align}
\label{eq:loss1}
\ell(X,Y; \lambda) &= (1-a_\lambda(X)) (1-m(X; f(X), Y)).
\end{align}
A trivial policy that always abstains would result in a zero loss. However, an interesting policy would also have a small abstention rate. The quality of a policy that can abstain is controlled by: (i) the \emph{risk} $R(\lambda) = \E[\ell(X,Y;\lambda)]$ of producing an incorrect answer on a new query, and (ii) the \emph{rate of abstention} $T(\lambda)=\E[a_\lambda(X)]$, where the expectations are taken over a query-response pair $(X,Y)$ distributed according to $\cD$.

To balance these quantities, we are interested in finding the abstention threshold $\lambda$ resulting in the smallest number of abstentions for a given risk tolerance $\alpha>0$:
\begin{equation}
    \label{eq:rcam}
    \argmin_{\lambda\in\Lambda} T(\lambda)\,,\qquad \text{subject to}\,\qquad R(\lambda) \le \alpha \;.
\end{equation}
Since the abstention rate $T(\lambda)$ is a non-decreasing function of $\lambda$, this is equivalent to finding the smallest $\lambda$ for which $R(\lambda) \le \alpha$; we denote this optimal threshold by $\lambda^*$.

To solve this problem approximately, we assume that we are given a calibration dataset 
\[
D_n = \{(X_1,Y_1),...,(X_n,Y_n)\}\subset \cX\times\cY \,, 
\]
which is a collection of ground truth query-response pairs. We also assume that given a new test point $(X,Y)$ sampled from the true data distribution $\cD$, and that $\{(X,Y),(X_1,Y_1),...,(X_n,Y_n)\}$ are exchangeable\footnote{Jointly distributed random variables $Z_1, \ldots, Z_n$ are exchangeable if for every permutation $\pi$ of $[n]$, $P(Z_1, \ldots, Z_n) = P(Z_{\pi_1}, \ldots, Z_{\pi_n})$.} (which is a generalization of the assumption that they were all selected independently from $\cD$). We will use the calibration dataset $D_n$ to design our abstention policy, that is, to find a $\widehat{\lambda}$ such that we can guarantee $R(\widehat\lambda) \le \alpha$ with high probability, based on $D_n$. Notice that the calibration dataset is much smaller than the training dataset that is used to train the LLM. Before discussing how $\lambda$ is optimized (which is presented in Section~\ref{sec:conformal-abstention}), we first discuss potential choices for the classifier $f$ and the score function $g$ in our context.

\subsection{Choice of the score function $g$ and the classifier $f$}
\label{sec:score-f}

In this section, we discuss the choice of the score function $g$ and the classifier $f$. Let $k$ be an integer. We augment each question-answer pair $(X_i,Y_i)$ with $k$ samples $Y^1_i,\dots, Y^k_i$ generated from the LLM given a query $X$. So a datapoint in the calibration data will be of the form $(X_i, Y_i, Y^1_i,\dots, Y^k_i)$. Notice that in many use cases of \ac{LLM}s, we already generate multiple responses for a given query and output a response based on a number of criteria. So we are not adding a computational overhead here by demanding the existence of $k$ responses. We can choose $k$ to be any number of responses the LLM already generates.

We consider two score functions. The first, called \emph{match count}, is defined with respect to a contextual similarity function $s:\cX\times\cY\times\cY\rightarrow \real$ (that might be different than the match function $m$) and is parameterized by a positive scalar parameter $\beta$. By default, we suggest using LLM prompting to measure similarity of text outputs, but other similarity functions could also be used. For a query $X$, generated responses $Y^1, \dots, Y^k$, and a parameter $\beta$, let the score of response $Y^i$ be the number of other responses that are similar to $Y^i$, that is, $|\{j\neq i\,:\,s(X; Y^i, Y^j) > \beta\}|$. The default response $f(X)$ is a response with the largest score, and the score $g(X)$ is the score of $f(X)$. As explained in the previous section, the policy abstains if the score is below $\lambda$. Otherwise, the policy returns the response $f(X)$. 

Similarly as reported in the literature \citep{MLG2023,KuhnARXIV2023}, we have observed that using LLM prompting as the similarity function works well in practice. Computing the score function then requires $O(k^2)$ extra inferences, which adds significant computational overhead. There are multiple cheaper alternatives. One cheaper alternative is to get similarity of each response with all other responses in a single prompt. This alternative still performs well in practice while being much faster to compute. An even more interesting alternative, called \emph{expected match count}, is the following: for each response $Y^i$, ask the LLM in a single query how many matches exist among other responses $\{j\neq i\,:\, Y^j\}$. Then the score is the expected match count $g(X)=\sum_{i=1}^k q(``i" \mid X) \, i$, where $q$ is the probability of token $``i"$ according to the LLM. In addition to being computationally inexpensive, this score can take any values in interval $[0,k]$, which allows for a more fine-grained and improved optimization. On the other hand, computing this score requires access to the log-probabilities of the LLM, and is not a black-box solution.

Finally, the simplest alternative is to choose $f$ to be the greedy (zero-temperature) output of the LLM (denoted, say, by $Y^1$), and the score of this prediction is the number of similar responses in the randomly selected samples $Y^2,\ldots,Y^k$, as defined either by the match count or the expected match count above. This approach reduces the computation cost of the comparisons by a factor of $k$, and we refer to it as the \emph{greedy} version of the methods.

\subsection{Choice of the match function $m$}
\label{sec:match-function}

Match functions can be naturally derived from similarity scores: two responses match if their similarity score is large enough (i.e., larger than a given threshold). A popular similarity score function, usually defined in term of a response and a true label, is the F1 score~\citep{joshi2017triviaqa, devlin2019bert}, which is calculated as 
\[
F1 = \frac{2 \times \text{precision} \times \text{recall}}{\text{precision} + \text{recall}}\,,
\]
where precision is the percentage of the response words that appear in the label sentence, and recall is the percentage of the label words that appear in the response sentence. When the labels are short sentences, as is the case in our experiments, we can obtain more reliable results using only the recall score. For experiments on the TriviaQA dataset~\citep{joshi2017triviaqa}, with short answers and labels, we use the recall score to evaluate different methods. 

Both the F1 and the recall scores however are poor choices when LLM answers are longer and can be expressed in many forms. For experiments conducted on the Temporal Sequences dataset~\cite{srivastava2023beyond} which we consider in the following, responses are sometimes long texts, and so we use LLM-prompting to decide if the generated answer and the label match, using the same similarity metric as before, by asking the LLM to measure similarity of two texts given the question on a scale of $1-10$. If the score is above a pre-specified threshold, the generated text is considered correct (or a match). The same conformal risk control procedure (discussed in the next section) can be used to verify the validity of this match function choice. The details of the calibration of the match function are presented together with the descriptions of the experiments in \Cref{sec:experiments}.

In the next section, we discuss tuning of the abstention policy and the match function based on the calibration set.

\section{Conformal abstention}
\label{sec:conformal-abstention}

Given the calibration dataset, we want to construct a postprocessing procedure that guarantees that the resulting composite policy (which depends on the calibration data, and hence, is random) is an approximately optimal solution for problem \eqref{eq:rcam}. 

Notice that the loss function is non-increasing in $\lambda$: for $\lambda_1 \le \lambda_2$, if $a_{\lambda_1}(X)=1$, then $a_{\lambda_2}(X)=1$ and both parameters have zero loss. On the other hand, $f$ does not depend on $\lambda$, 
and hence the loss of $\lambda_2$ is smaller than or equal to the loss of $\lambda_1$. Given that calibration data and the test point are exchangeable while the loss function $\ell$ is non-increasing in $\lambda$, then we can use the \emph{\ac{CRC} framework} of \citet{AngelopoulosARXIV2022} to tune $\lambda$.
In particular,
define the average loss
\[
L_n(\lambda) = \frac{1}{n}\sum_{i=1}^n \ell(X_i,Y_i; \lambda) \;
\]
and let
\begin{align}
\label{eq:crc-1}
\widehat\lambda_n = \inf\left\{\lambda\,:\, \frac{n}{n+1}  L_n(\lambda) + \frac{1}{n+1} \le \alpha \right\} \;.
\end{align}
Then, it holds that~\citep{AngelopoulosARXIV2022}
\begin{align}
\label{eq:crc-2}
\E[R(\widehat\lambda_n)] = \E[\ell(X, Y; \widehat\lambda_n)] \le \alpha \;.
\end{align}
The expectation in \eqref{eq:crc-2} is over calibration data as well as the test point.
For completeness, the proof is given in \Cref{sec:crc-expectation}.

The above guarantee is non-trivial: standard confidence-interval-based methods would lead to a solution that uses a more conservative padding of order $O(1/\sqrt{n})$ instead of the smaller $O(1/n)$ padding used in~\eqref{eq:crc-1}. We will discuss this alternative approach, called \ac{RCPS}, in Section~\ref{sec:LTT}.

\subsection{Simple high probability amplification of the CRC procedure}
The \ac{CRC} formulation of \citet{AngelopoulosARXIV2022} given by \eqref{eq:crc-1} and \eqref{eq:crc-2} holds only in expectation over the calibration data.
To have reliable decision making, in practice one typically desires to have confidence guarantees that hold with high probability over the samples. In this and the following section we present several methods that come with high-probability guarantees; some of these approaches will be compared experimentally in \Cref{sec:experiments}.

Clearly, under the assumption that the loss function is non-negative, one can convert a \ac{CRC} guarantee in expectation to a guarantee in probably, for instance through Markov's inequality.
Namely, \eqref{eq:crc-1} and \eqref{eq:crc-2} imply that $\P(R(\hat \lambda) \geq \alpha / \delta) \leq \delta$ for any failure probability $\delta$.
This result is rather weak as it does not hold with high probability.
However, interestingly, we can further improve upon Markov's inequality without extra assumptions through an amplification (or boosting) argument, at the expense of data splitting.
Unlike Markov's inequality, such inequality provides a high-probability guarantee for the \ac{CRC} procedure, however it is looser by a constant factor than the high-probability inequalities we will consider for \ac{RCPS} in the coming section.
\begin{proposition}
  \label{prop:hp}
  Assume that loss function $\lambda \mapsto \ell(z; \lambda)$ is non-increasing, right-continuous, and bounded within $[0, B]$ for any $z$. Assume data $D_n$ is composed of i.i.d. samples. 
  Let the failure probability be $\delta \in (0,1)$ and the desired error be $\alpha > 0$.
  Then, consider an arbitrary partition
  $D_n = (S_1, \ldots, S_K)$ where $K = \lceil \ln(1/\delta) \rceil$,
  and let
\begin{align*}
  \hat \lambda_i = \inf\cbr{\lambda \in \Lambda ~:~ \frac{n}{n + \ln(\frac{1}{\delta})} \, L(S_i; \lambda) + \frac{B \ln(\frac{1}{\delta})}{n + \ln(\frac{1}{\delta})} \leq \alpha}~.
\end{align*}
Then, for $\hat \lambda^* = \max_{i \in [K]} \hat \lambda_i$,
\begin{align*}
  \Prob\pr{R(\hat \lambda^*) \leq e \, \alpha} \geq 1 - \delta~.
\end{align*}
\end{proposition}
In particular, the constant $e$ can be replaced by $c > 1$ while also replacing $\ln(\frac{1}{\delta})$ by $\frac{\ln(1/\delta)}{\ln(c)}$.
\begin{proof}[Proof of \Cref{prop:hp}]
  The proof is based on \eqref{eq:crc-1} and \eqref{eq:crc-2}.
  In particular, for some $\gamma > 0$, consider the following probability:
  \begin{align*}
    \P\pr{\bigwedge_{i=1}^K \quad R(\hat \lambda_i) > \frac{\E[R(\hat \lambda_i)]}{\gamma}}=
    \prod_{i=1}^K \P\pr{R(\hat \lambda_i) > \frac{\E[R(\hat \lambda_i)]}{\gamma}}
    \leq
    \gamma^K,
  \end{align*}
  where we note that the probability factorizes since $(\hat \lambda_i)_i$ are fitted using independent samples
  and the last inequality follows from Markov's inequality.
  Now, choosing $\gamma = 1/e$, we have that
  \begin{align*}
    \P\pr{\exists i \in [K] \quad R(\hat \lambda_i) \leq e \E[R(\hat \lambda_i)] } \geq 1 - e^{-K}\Longrightarrow \qquad
    \P\pr{\exists i \in [K] \quad R(\hat \lambda_i) \leq e \, \alpha } \geq 1 - e^{-K}
  \end{align*}
  where the last step follows by \eqref{eq:crc-2} with $S' = S_i$ and so $|S'| = n / K = n / \lceil \ln(1/\delta) \rceil$.

  Finally, since our result so far only guarantees the existence of one $\hat \lambda_i$ that succeeds with high probability, we use the property that $L(\cdot)$ is non-decreasing to claim that
  \begin{align*}
    \P\pr{R(\max_{i \in [K]} \hat \lambda_i) \leq e \, \alpha }
    = \P\pr{\min_{i \in [K]} R(\hat \lambda_i) \leq e \, \alpha }\geq \P\pr{\exists i \in [K] \quad R(\hat \lambda_i) \leq e \, \alpha }~.
  \end{align*}
The statement then follows.
\end{proof}

\subsection{Bounding $|\lambda^* - \widehat{\lambda}_n|$}

Although the \ac{CRC} procedure ensures that the constraint inequality in \eqref{eq:rcam} is satisfied, it provides no guarantees on how the abstention rate of the solution $\widehat{\lambda}_n$ deviates from the optimal abstention rate, i.e. a bound on $|T(\lambda^*) - T(\widehat{\lambda}_n)|$.  

Assuming that the abstention rate $T$ and risk $R$ are differentiable functions of the threshold $\lambda$,
under the assumptions of Theorem~2 of \citet{AngelopoulosARXIV2022}, we have
\begin{align*}
    \mathbb{E}[R(\hat{\lambda})] \geq \alpha - \frac{2}{n+1} \geq R(\lambda^*) - \frac{2}{n+1} \implies R(\lambda^*) - \mathbb{E}[R(\hat{\lambda})] \leq \frac{2}{n+1}.
\end{align*}
Let $\gamma = \sup_\lambda \frac{dT(\lambda)}{dR(\lambda)}$ and assume it is finite.
Then,
\begin{align*}
T(\lambda^*) - \mathbb{E}[T(\hat{\lambda})] &= \mathbb{E}[T(\lambda^*) - T(\hat{\lambda})]\\
&\leq \mathbb{E}\left[\gamma(R(\lambda^*) - R(\hat{\lambda}))\right]\\
&\leq \frac{2}{n+1}\gamma \;.
\end{align*}
\section{\acl{RCPS}}
\label{sec:LTT}

Motivated by the need for high-probability guarantees over the calibration data, \citet{AngelopoulosARXIV2021c} also introduced another family of methods, called \ac{RCPS}. Indeed, for a loss bounded by one as we have, we can get for $\widehat \lambda_n$ obtained using conformal risk control the following high-probability result:
\begin{equation}\label{eq:probabound}
    \mathbb{P}\Big(\E[ \ell(X, Y; \widehat \lambda_n) \mid D_n ] \leq \alpha +c(\delta,\alpha,n)\Big)\geq 1-\delta,
\end{equation}
for $c(\delta,\alpha,n)=u(\delta,n)-v(\alpha,n)$, where the probability is over the calibration set $D_n$ and
\begin{align*}
u(\delta,n)&=\sqrt{\frac{\log(\frac{1}{\delta})}{2n}} \quad \text{ and } \quad v(\alpha,n)=\frac{1-\alpha}{n}.
\end{align*}
This follows from the fact that $\hat \lambda_n$ can be rewritten as
\begin{align*}
\widehat{\lambda}_n&=\inf\left\{\lambda\,:\, \frac{n}{n+1}  L_n(\lambda) + \frac{1}{n+1} \le \alpha \right\} \nonumber\\
&=\inf \Bigl\{\lambda: L_n(\lambda)+v(\alpha,n) \leq \alpha \Bigr\} \nonumber\\
&=\inf \Bigl\{\lambda: L_n(\lambda)+ u(\delta,n) \leq \alpha+ u(\delta,n)-v(\alpha,n) \Bigr\} \nonumber\\
&=\inf \Bigl\{\lambda: L_n(\lambda)+ u(\delta,n) \leq \alpha+ c(\delta,\alpha,n)\Bigr\} \label{eq:monotonic-no-union-2}
\end{align*}
This identity shows that applying conformal risk control is equivalent to applying the distribution-free \ac{RCPS} procedure of \citet{bates2021distribution} for a Hoeffding \ac{UCB} on the empirical risk at level $\alpha+c(\delta,\alpha,n)$. It then follows from Theorem~2 of \citet{bates2021distribution} that \eqref{eq:probabound} holds. 

However,
the \ac{RCPS} approach is more general than \ac{CRC} as it is applicable even if the loss function is non-monotonic.\footnote{When the loss function is non-monotonic we can no longer rely on arguments as in \eqref{eq:probabound}, however we can still apply confidence bounds discussed here by making them hold uniformly over a finite parameter set $\Lambda$ through the union bound argument. In such case, $\delta$ is replaced by $\delta/|\Lambda|$.}
Let $\delta\in (0,1)$ be a failure probability. We want to choose $\widehat \lambda$ to ensure that
\[
\Prob\Big(  \E[ \ell(X, Y; \widehat\lambda) | D_n ] \le \hat R_{\mathrm{ub}}(\widehat\lambda) \Big) \ge 1-\delta \;.
\]
Here, $\Prob$ is again over the random calibration set $D_n$.
In the following we consider several upper confidence bounds for \ac{RCPS}, some of which were already discussed by \citet{bates2021distribution}.
\paragraph{Baseline confidence bounds}
Among \ac{RCPS} methods
we first consider the \emph{empirical Bernstein inequality}~\citep{audibert2007tuning,maurer2009empirical}.
In this case, the upper bound is computed as
\begin{align*}
  \hat R_{\mathrm{ub}\text{-}\mathrm{bern}}(\lambda) =
  L_n(\lambda)
  +
  \sqrt{\frac{2 \widehat{\mathrm{Var}}(\lambda) \, \ln(\frac{2}{\delta})}{n}}
  +
  \frac73 \, \frac{\ln(\frac{2}{\delta})}{2(n-1)}
\end{align*}
where $\widehat{\mathrm{Var}}(\lambda)$ is a sample variance of losses computed with parameter $\lambda$.

Since we are working with Bernoulli losses, we evaluate the \emph{Hoeffding-Bentkus inequality}, one of the tightest known bounds for such losses.
Computation of the bound relies on the following function (of $t, p \in [0,1]$),
\begin{align*}
  \varepsilon_{\mathrm{hb}}(t, p) = \min\cbr{e^{-n \, \mathrm{kl}(t, p)},\, \P(\mathrm{Bin}(n, p) \leq \lceil n t \rceil)},
\end{align*}
where $\mathrm{Bin}(n, p)$ is a binomial random variable with parameters $n \in \mathbb{N}$ and $p \in [0,1]$
and $\mathrm{kl}(p,q) = p \ln(p/q) + (1-p) \ln((1-p)/(1-q))$ is the relative entropy between two Bernoulli distributions with success probabilities $p$ and $q$, respectively.
Then, the upper bound is given by solving a simple optimization problem
\begin{align*}
  \hat R_{\mathrm{ub}\text{-}\mathrm{hb}}(\lambda) = \sup\cbr{ p \in [0,1] ~:~ \varepsilon_{\mathrm{hb}}(L_n(\lambda), p) \geq \delta }.
\end{align*}

Finally, we consider the so-called \emph{Bernoulli relative-entropy inequality}, a.k.a. the `little kl' inequality (see, for instance, \citealp{maurer2004note}).
Here the upper confidence bound is computed by solving the simple optimization problem
\begin{align*}
  \hat R_{\mathrm{ub}\text{-}\mathrm{kl}}(\lambda) = \sup\cbr{ p \in [0,1] :~ \mathrm{kl}(L_n(\lambda), p) \leq \frac{\ln(\frac{\sqrt{n}}{\delta})}{n}}.
\end{align*}
\citet{bates2021distribution} mentions another, so-called Waudby-Smith-Ramdas (WSR) inequality for the case of non-binary losses, which is tighter for such cases since it adapts better to the variance.
This inequality belongs to the family of concentration inequalities derived through regret analysis of online betting algorithms, first proposed by \citet{JunO19}.
In fact, it was recently shown that WSR inequality is looser than another inequality from this family~\citep{orabona2023tight}, and which notably, for Bernoulli distributions coincides with the Bernoulli relative-entropy inequality considered above.

\section{Calibrating the match function $m$}
\label{sec:match-calibration}
As described in \Cref{sec:match-function}, it is hard to identify if two responses to a query (e.g., one generated by the LLM and another being the ground truth answer) are the same, and hence devising a good match function (to be used in computing the loss $\ell$) is a non-trivial problem. As explained before, we consider score-based match functions . Then, given a similarity score $s:\cX\times\cY\times\cY\rightarrow \real$, it is natural to define $m$ as a thresholded version of the score function: 
\[
m(X, Y', Y) = m_\beta(X, Y', Y) = \one\{s(X,Y,Y') \ge \beta\}
\]
where for any event $E$, $\one\{E\}$ denotes its indicator function and $\beta$ is a threshold to be chosen. Note that although we use the same notation for the similarity function and its threshold as in the definition of the score function in \Cref{sec:score-f}, these are not necessarily the same. 

In this section we assume that $s$ is given (in the experiments we will use different option, such as recall or LLM self-prompting, discussed in \Cref{sec:match-function}), and the goal is to select a threshold $\beta$ so that the match function $m$ reflects the ground truth as much as possible (given $s$).
We can do this based on another calibration set, again, with a slight inconsistency in the notation, denoted by $(X_1,Y_1',Y_1),\ldots,(X_n,Y'_n,Y_n)$, where, for all $i$, $(X_i,Y_i)$ are ground-truth question-answer pairs sampled independently from the data distribution $\cD$, and $Y'_i$ is the model's response to query $X_i$. Whether $Y_i$ and $Y'_i$ agree has to be checked manually, so the size $n$ of this calibration set can be quite small in practice.

If the quality of the responses is monotone in $s$, that is, if $s(X,Y',Y) < s(X,Y'',Y)$ means that $Y''$ is a better response to $X$ than $Y'$ (as suggested by the ground truth response $Y$), then one can use any of the methods discussed in the previous sections, such as \eqref{eq:crc-1}, to select a threshold $\beta$ to get a guarantee on the error the match function makes when comparing responses to the ground truth; here we can define $\ell(X_i,Y'_i,Y_i)$ to be 0 if the match function is correct about comparing $Y_i$ and $Y'_i$ and 1 otherwise.

However, none of our similarity function candidates are monotone, as typically a too high threshold becomes too conservatives and may classify some correct responses $Y'_i$ as incorrect, while a too low threshold may result in incorrect answer classified as correct. Nevertheless, we present next a procedure which, using an upper bound on the performance of $m$, allows us to calibrate the threshold $\beta$ with theoretical guaranties.

Let $C$ be the number of incorrect LLM responses (i.e., when the LLM's response does not match the label according to the human rater) for our $n$ calibration samples. This is the true performance measure. Let $L_2$ denote the number of times the LLM's response is different from the label according to the match function $m$. This is the performance measure that we report when we use $m$ as a surrogate to the true loss. Next we show how proper calibration of $\beta$ can ensure that $L_2$ is an approximate upper bound on $C$, and hence reporting errors based on $m$ can be used to upper bound the true error rate. 
Let $L_1$ denote the number of times the LLM's response is different from the corresponding label, but is classified as the same according to the match function. Then clearly
\[
C \le L_1 + L_2.
\]
While the dependence of $L_2$ on the threshold $\beta$ can be arbitrary in general, it is easy to see that $L_1$ is a monotone decreasing function of $\beta$ (setting a higher threshold $\beta$ either keeps $m(X,Y,Y')$ unchanged or changes it from 1 to 0), allowing the application of conformal prediction to set $\beta$ with theoretical guarantees on the behavior of $L_1$ on new data using the calibration dataset $(X_1,Y'_1,Y_1),\ldots,(X_n,Y'_n,Y_n)$.

For example, setting the value of $\beta$ according to the conformal prediction rule \eqref{eq:crc-1} as
\[
\hat{\beta} = \inf \left\{ \beta: \frac{n}{n+1} \sum_{i=1}^n \big(1-m_\beta(X_i,Y'_i,Y_i)\big) + \frac{1}{n+1} \le \alpha\right\}
\]
guarantees that 
\[
\E[L_1] \le \alpha
\]
on new test data, which implies that on expectation $L_2 + \alpha$ is an upper bound on the number of errors the LLM makes. Note that since $L_2$ is evaluated using the (calibrated) match function $m$, at test time we can use a lot of data, making the measured value of $L_2$ arbitrarily close to its expectation $\E[L_2]$ (we can also use any of the confidence bounds from \Cref{sec:LTT} to upper bound their difference), we can guarantee that the expected number of errors $\E[C]$ made by the LLM satisfies
\begin{equation}
\label{eq:m_tuning}
\E[C] \le \E[L_2] + \alpha \le L_2 + \alpha + \epsilon,
\end{equation}
where $\epsilon \ge \E[L_2]- L_2$ is an upper bound on the difference of $L_2$ and its expectation, which can be made arbitrarily small. Selecting a calibration method which comes with a high-probability guarantee would yield a high-probability version of \eqref{eq:m_tuning}.

\section{Related work}

Uncertainty quantification of machine learning methods is a large and active area of research. We only discuss prior papers that are closely related to our work.

\subsection{Selective classification}

The problem that we study is a case of selective classification~\citep{EW2010,GE2017,lin2022scrib}. Given a classifier, a training set, a confidence parameter, and a desired risk bound, the objective of 
\citet{GE2017} is to design an abstention policy such that the risk is bounded by the desired bound with high probability. They normalize loss by decision rate (one minus abstention rate), which makes loss non-monotonic. \citet{GE2017} propose a binary search procedure. However, given the non-monotonicity of the loss function, the binary search procedure is not guaranteed to find a solution that satisfies the risk condition.

\citet{KJL2020} study selective question answering when the test point might be out-of-domain.  
Selective classification methods are closely related to the \ac{RCPS} approach that we discussed in Section~\ref{sec:LTT}.

\subsection{Abstention in \ac{LLM}s}

There has been a number of recent papers that study abstention in \ac{LLM}s. We only cover approaches that use a pre-trained model, and not those based on fine-tuning \ac{LLM}s. These papers usually consider general metrics for their methods, such as the area under the curve, and do not provide any practical guidance on how to actually choose an abstention policy, which is one of our main contributions. Given a risk tolerance $\alpha$, the policy that these papers implicitly suggest chooses a policy parameter that leads to $\alpha$ loss; while this method comes with no theoretical guarantees, we consider it as a baseline for our calibration methods in our experiments. 

\citet{CZGEDE2023} investigate a number of score functions in designing an abstention mechanism: (i) a likelihood-based score; (ii) using sampling repetition and  counting how many times the sampled output matches exactly (after making the response lower case and removing punctuation) the greedy (zero-temperature) output; (iii) computing sampling diversity defined as the fraction of non-unique answers; and (iv) using self-verification by checking the probability given by the model if the greedy answer is correct. They report that their approach (ii) is generally the best. However, since it considers exact match of the responses, its applicability is limited to short responses only (otherwise exact matching almost never happens in practice).
Furthermore, the resulting abstention policy has no theoretical (statistical) performance guarantee, unlike the one we propose here.

\citet{MLG2023} study a black-box approach to detecting hallucinations by generating multiple responses, and measuring similarity of a reference response and the set of generated responses. 
They consider various measures of similarity, including LLM self-prompting. 
In their experiments, the method that generates multiple responses and uses LLM self-prompting for similarity calculations outperforms other baselines including the one using log-probability scores.
Although the overall approach in this paper is conceptually similar to ours, their self-prompting method only compares responses without contexts, resulting in inferior similarity measures. Similar methods have been studied by \citet{LTS2023}. Furthermore, as discussed at the beginning of this section, the choice of an actual abstention policy is not discussed in either of these papers.

\citet{KuhnARXIV2023} study uncertainty quantification of \ac{LLM}s. Similarly to our work,they propose generating $k$ responses, and clustering them based on their contextual similarities evaluated using a smaller language model. Then they investigate the application of semantic entropy to score model uncertainty.\footnote{Note that their formula (4) for estimating semantic entropy is incorrect, as it gives uniform weight to all clusters.} As entropy measures the uncertainty of the whole output distribution, this method does not seem to be directly applicable to decide between using a given response (e.g., the zero-shot response) or abstain.

\citet{wang2023selfconsistency} study reasoning with LLMs and propose generating a set of `reasoning paths' instead of a final answer.
Here reasoning paths are generated by prompting the model to provide intermediate reasoning steps used to arrive at the answer.
Instead of greedily choosing the `best' answer according to some criterion of the associated reasoning path, the paper proposes to select the most consistent answer.
This approach is complementary to the one we consider in this paper, and in principle, the match function and the score function can be computed using reasoning paths.

\subsection{Using token probabilities to quantify uncertainty}

A popular approach to quantify uncertainty is based on using (normalized) log-probabilities of responses. \citet{KCAHD2022} show that LLMs are well-calibrated at the token level on multiple-choice question-answering tasks when the prompts are in an appropriate format. However, the quality of log-probability scores quickly degrades as the model generates longer texts~\citep{CZGEDE2023,MLG2023,KuhnARXIV2023}. 

\subsection{Asking language models to quantify uncertainty (self-verification)}

\citet{KCAHD2022} propose using LLM self-prompting to measure a model's uncertainty in its responses. More specifically, for a given query, a number of responses are generated, and then the model is queried if the responses are correct. For this query, the log-probability of ``True" is returned as a measure of uncertainty. Related approaches are studied by \citet{mielke2022reducing}. However, \citet{MLG2023} and \citet{KuhnARXIV2023} report that LLM self-verification is not as effective as sampling-based methods (i.e., methods using multiple responses) in quantifying model uncertainty.

\subsection{Applications of conformal prediction in quantifying uncertainty in \ac{LLM}s}

Conformal prediction has been used for quantifying uncertainty in LLMs, but we are not aware of any works that employ conformal prediction in designing an abstention mechanism. \citet{QuachARXIV2023} use conformal prediction to construct confidence sets of text outputs that contain an acceptable answer with a high probability, based on a calibration mechanism applied to log-probability scores. \citet{RavfogelACL2023} propose using conformal prediction to calibrate parameter $p$ in nucleus (top-$p$) sampling. \citet{RDBSTetal2023} consider a multiple-choice-style LLM planning, and use conformal prediction to quantify uncertainty of LLM-based planners.

\subsection{Other uncertainty-quantification methods in deep learning and 
\ac{LLM}s}
\emph{Ensemble methods} are based on the classical idea of bootstrap for confidence estimation \citep{tibshirani1993introduction} where multiple estimators for the regression function, each computed on a perturbed version of the data (e.g. by drawing samples from the empirical distribution over data), are combined.

The empirical distribution of the resulting estimates is then used to construct confidence intervals.
While many of these methods can be interpreted as sample-based approximations to Bayesian methods, model-hyperparameter selection (e.g., scale of perturbations, learning) for ensemble methods is typically done using a validation on holdout data (a subset of the training data).
Many recent papers have studied ensemble methods in the context of deep learning and reinforcement learning \citep{OV2015,LPB2017,MG2020}.
In the context of \acp{LLM}, the methods require training multiple language models, which is very expensive. \citet{osband2023epistemic} introduces epistemic neural networks (epinets), which approximate ensemble methods by training a single network with an artificially injected (controlled) source of randomness.
\citet{rabanser2022selective} proposes to use intermediate model checkpoints to quantify the uncertainty of the final model in its responses.
While these approaches aim to mimic the bootstrap procedure during prediction, their validity is not justified by theoretical considerations, and hence remain heuristic approximations.
\section{Experiments}
\label{sec:experiments}

Our experiments aim to verify three hypotheses: (i) conformal abstention done through \ac{CRC} and \ac{RCPS} is able to mitigate hallucinations as measured by loss \eqref{eq:loss1}, while maintaining a low abstention rate; (ii) the loss \eqref{eq:loss1} is a reasonable measure of detecting hallucinations; and (ii) for longer responses, defining scores using LLM similarity prompting is more effective than the ones based on log-probabilities.

\textbf{Datasets.}
We evaluate our approach on two publicly available question-answering datasets: Temporal Sequences (a dataset from the BIG-bench benchmark of \citealp{srivastava2023beyond}) and TriviaQA \citep{joshi2017triviaqa}.
TriviaQA predominantly contains short answers while Temporal Sequences contains several long answers as well.
We hypothesise that some commonly used scores, such as log-probabilities predicted by the model will not yield a good performance on long answers, and therefore we expect that the calibration procedure combined with log-probability scores will perform worse than the calibration procedure with our proposed match-scores on  Temporal Sequences.

\textbf{Calibration/test splits.}
In each experiment, 20\% of the data is used for testing (holdout sample). Each experiment is performed on subsamples of calibration sets of increasing sizes; moreover, each subsample is drawn with replacement 10 times. We report the resulting average test losses and their standard deviations.
We also report the median for the parameter $\lambda$ of our methods. 

\textbf{Language model.} We use a
Gemini Pro
model \citep{geminiteam2023gemini} to generate outputs and scores.

\textbf{The match function $m$.}
We use a similarity-score-based match functions to compute the loss $\ell$, as described in \Cref{sec:match-function}, and calibrate its threshold according to \Cref{sec:match-calibration}. Thus, first we have to choose a similarity score function with a corresponding threshold.
First we discuss the TriviaQA dataset, which contains short answers. For such cases, typically the F1 score is used in the literature~\citep{joshi2017triviaqa, devlin2019bert}. However, to better accommodate the case that the response may be long and the answer (label) is very short, which significantly reduces the F1 score, we rather consider recall as the similarity score in the experiments. To select the threshold, we 
uniformly sampled 100 question-answer pairs that were not used for calibration or testing, and manually inspected the similarity of the generated response and the true answer.
With respective thresholds of 0.5 and 0.25, the recall and F1 scores make no mistakes, hence, in the experiments we used recall with threshold 0.5 as the match function for this dataset.
According to \eqref{eq:m_tuning}, this implies that any measurement of the error rate in testing with this match function would result in at most 1 percentage-point lower error in expectation than the ground truth.

Selecting an appropriate similarity function for the Temporal Sequences dataset is much harder because it has a large proportion of long answer, which makes the F1 and recall scores much less useful: with the same thresholds as above, on a random sample of 100 question-answer pairs, the F1 score resulted in 45 mistakes while and recall score ended up with 13, after manually checking the validity of the corresponding responses generated by the language model. Therefore, we decided to prompt the LLM to compute the similarity of the response and the true answer, using the similarity-seeking prompt presented in \Cref{app:experiments}. Then we computed the smallest threshold (which was $\hat{\beta}=7$ in this case) so that the number of errors in the 100 datapoints was 4 (as verified by manual inspection); according to \eqref{eq:m_tuning}, this guarantees that the (expected) error rate as measured by the resulting match function (i.e., using the $\hat\beta$-thresholded LLM self-prompting score) is at most 5 percentage-point lower than the true error rate.
Therefore, for the Temporal Sequences dataset we used the LLM self-prompting similarity score with threshold $\hat\beta=7$ to compute the match function. (Note that the same method resulted in 2 mistakes for the TriviaQA dataset.)

\textbf{Calibration methods.}
For calibrated methods with theoretical guarantees, we consider the CRC method (defined in \eqref{eq:crc-1}, and referred to as `Bound in expectation'),
and three variants of the \ac{RCPS} procedure, as described in \Cref{sec:LTT}, with \ac{UCB} given by $\hat R_{\mathrm{ub}\text{-}\mathrm{bern}}$ (referred to as `Emp. Bernstein'), $\hat R_{\mathrm{ub}\text{-}\mathrm{hb}}$ (referred to as `Hoeffing-Bentkus'), $\hat R_{\mathrm{ub}\text{-}\mathrm{kl}}$ (referred to as `Bernoulli KL').

For a given risk tolerance $\alpha$, a simple baseline abstention policy chooses the smallest parameter $\lambda$ that satisfies $L_n(\lambda)\le \alpha$. We do not however have a theoretical guarantee on the risk of this baseline, and as we will show, it might violate the risk condition with small calibration datasets. This method is referred to as `Baseline' in the experiments.  

Note that we do not include the high-probability amplification of CRC discussed in \Cref{prop:hp} in our experiments. 
The risk guarantee provided by this bound is $\P(R(\hat \lambda^*) \leq e \, \alpha) \geq 1 - \delta$, i.e. the bound is inflated by $e$ compared to \ac{RCPS} baselines. So, to properly compare it to other baselines we need to replace $\alpha$ by $\alpha/e$, which makes the guarantee quite conservative and results in a very high abstention rate.

\textbf{Score functions.} We consider calibration of three different scores using the above methods. The first two scores are the \emph{greedy} variants of the score functions proposed and described in Section~\ref{sec:score-f}: In both cases, we take the greedy (zero-temperature) response as the reference response and sample additional $k-1=10$ extra responses at temperature 0.9. Then we either (i) prompt the LLM for the similarity of the reference response and each of the extra responses, and obtain the number of matches between the reference response and the extra responses --- this is referred to as \emph{match count (m.c.)} in the results; (ii) prompt the LLM for the number of matches between the reference response and the extra responses at once, and calculate the expected number of matches using the log-probabilities assigned by the LLM to the responses ``1", ``2", $\ldots$ --- this is referred to as \emph{expected match count (e.m.c.)} in the results. 

The third scoring method is the simple baseline of the log-probability of the zero-temperature response. Another popular score in the literature is the normalized log-probability; however, we only report results with the log-probability score, as in our experiments it always performed at least as good as its normalized version.
This baseline is referred to as \emph{log-probabilities (l.p.)} in the results.

The prompts used in calculating the score functions and some data samples are described in Appendix~\ref{app:experiments}.

\begin{figure*}[!t]
  \begin{center}
       \includegraphics[width=\textwidth]{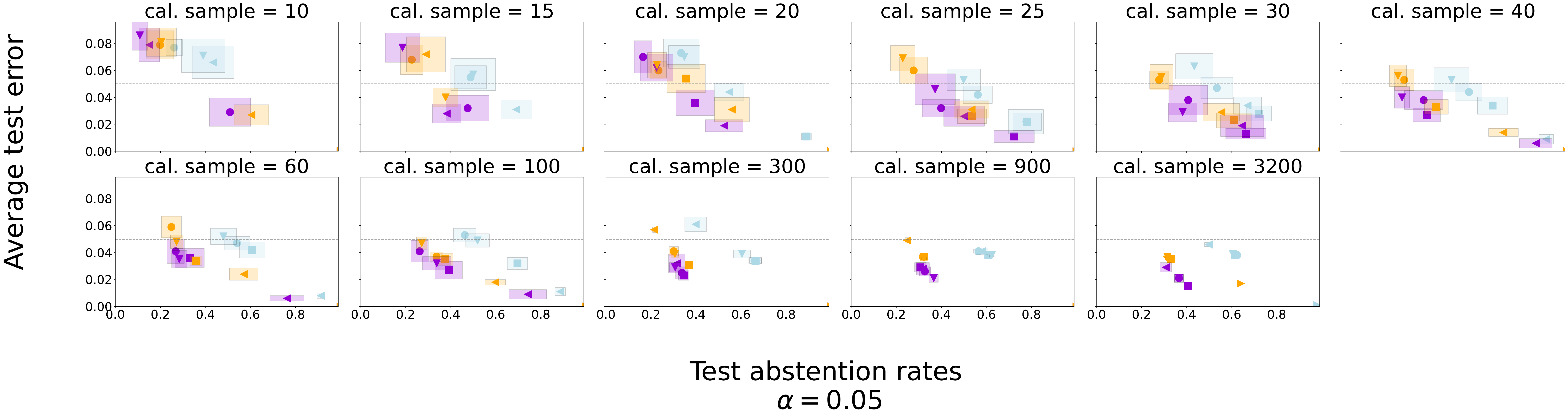}  
       
       \bigskip
       
       \includegraphics[width=\textwidth]{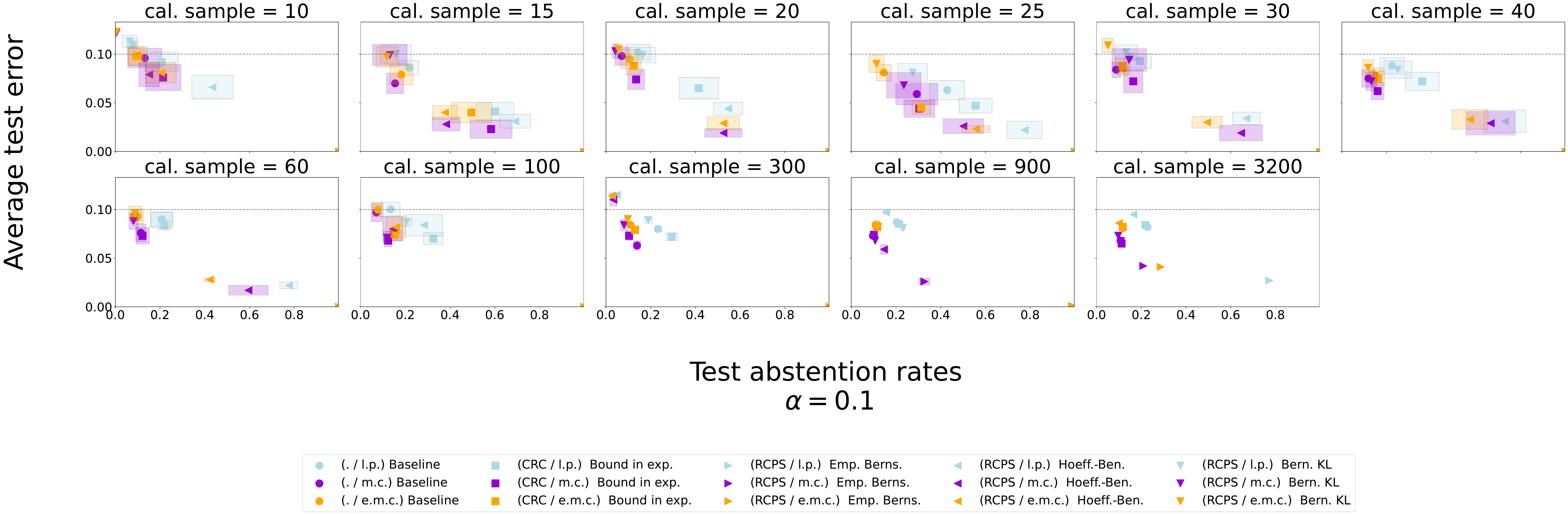}  
  \end{center}
  \caption{Abstention rates vs.\ average test losses on the Temporal Sequences dataset with $\alpha=0.05$ (top) and $\alpha=0.05$ (bottom) for score functions match count (m.c.), expected match count (e.m.c), and the log-probability (l.p.), and for various calibration methods (. denotes the baseline with no calibration). Box widths and heights represent 90\% confidence intervals with Gaussian approximation over abstention rates and average test errors, respectively. The dashed horizontal line represents the target risk bound $\alpha$. \label{fig:temporal_seq}}
\end{figure*}

\subsection{Results on the Temporal Sequences dataset}

We experimented with $4000$ question-answer pairs. The experiments were performed with two risk tolerance levels, $\alpha=0.05$ and $\alpha=0.1$, and we used $\delta=0.05$ failure probability for confidence intervals. 

The results of the experiments are reported in \Cref{fig:temporal_seq}, which shows the average test losses vs. abstention rates on the test sample for calibration datasets of various sizes (the exact numerical results are reported in \Cref{tab:Temporal_Sequences_test_loss_alpha_0.1,tab:Temporal_Sequences_abs_rates_alpha_0.1,tab:Temporal Sequences_cal_loss_alpha_0.1,tab:Temporal Sequences_median_lambdas_alpha_0.1,tab:Temporal Sequences (temp. $=0$ response)_test_loss_alpha_0.05,tab:Temporal Sequences (temp. $=0$ response)_abs_rates_alpha_0.05} in \Cref{sec:additional_tables}).
As expected, we can observe an inherent trade-off between the two metrics:
in particular, a larger abstention rate leads to a smaller test error; however, some methods and baselines exhibit better trade-offs.
For instance, by looking at \Cref{fig:temporal_seq} we can observe that for a sufficiently large calibration sample, it is evident that log-probability scoring performs considerably \emph{worse} regardless of which conformal prediction method (\ac{CRC}/\ac{RCPS}, confidence bound) is used.
At the same time, the proposed match count (m.c.) and expected match count (e.m.c.) proposed perform much better, and the difference between the \ac{CRC} and \ac{RCPS} methods is minimal.

We also observe that the Empirical  Bernstein calibration method is significantly worse than the others; this is expected since here we estimate Bernoulli random variables, and the other two bounds used in the RCPS methods are specialized for this case, unlike the Empirical Bernstein bound, which -- unlike the other two -- would be applicable for non-binary loss functions, as well.
We can also observe that the uncalibrated Baseline methods violate the risk conditions for smaller calibration datasets a bit more than other methods.

\begin{figure*}[!t]
  \begin{center}
       \includegraphics[width=\textwidth]{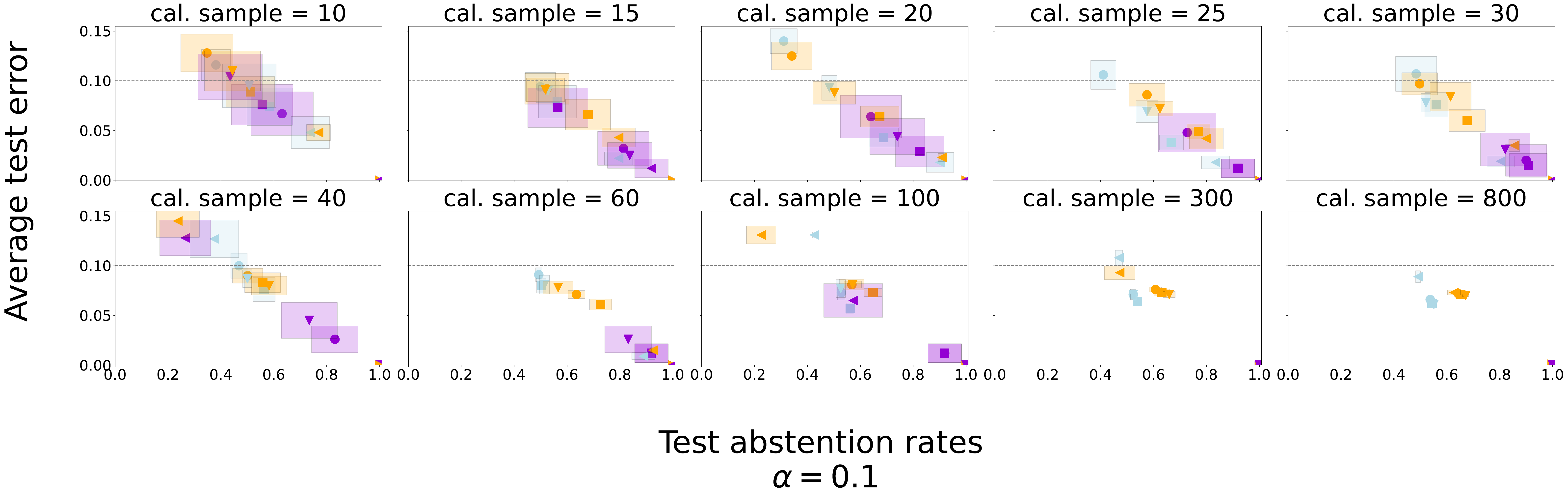}

       \bigskip
       
       \includegraphics[width=\textwidth]{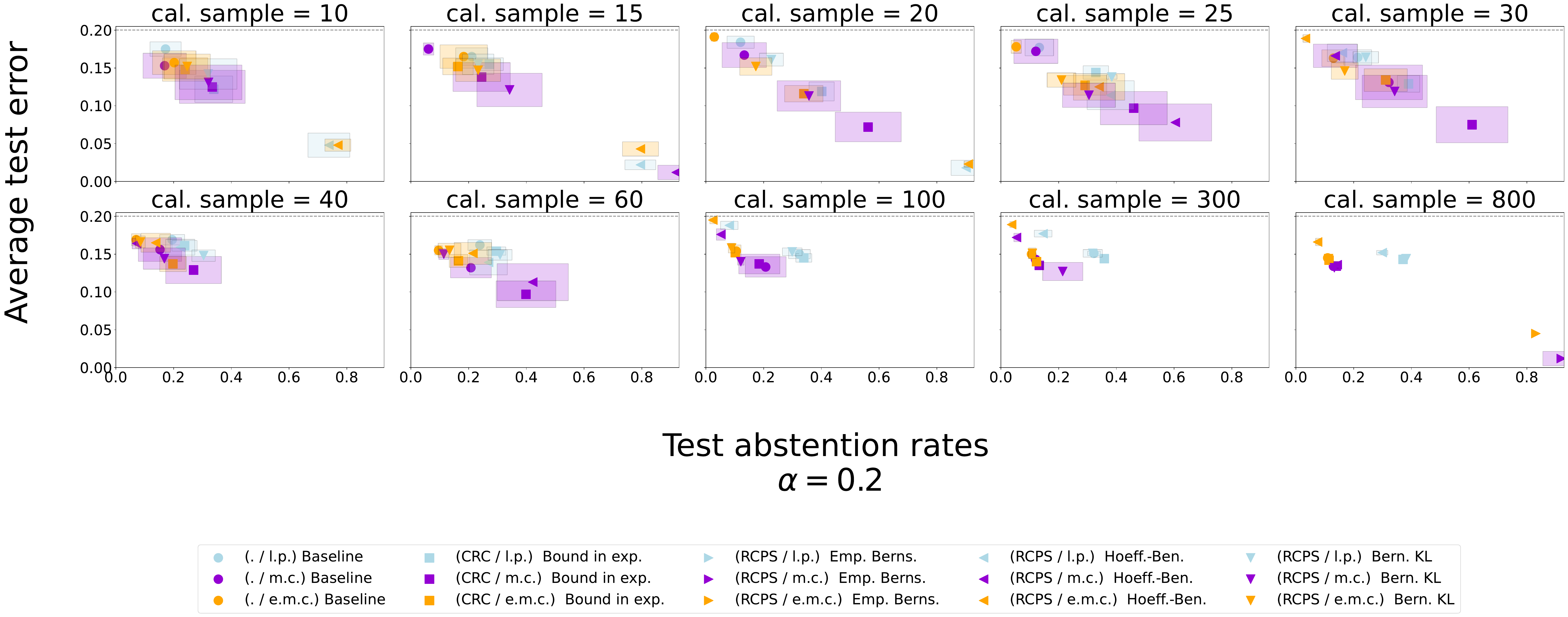}
  \end{center}
  \caption{Abstention rates vs.\ average test losses on the TriviaQA dataset with $\alpha=0.1$ (top) and $\alpha=0.2$ (bottom) for score functions match count (m.c.), expected match count (e.m.c), and the log-probability (l.p.), and for various calibration methods (. denotes the baseline with no calibration). Box widths and heights represent 90\% confidence intervals with Gaussian approximation over abstention rates and average test errors, respectively. The dashed horizontal line represents the target risk bound $\alpha$. \label{fig:triviaqa}}
\end{figure*}

\subsection{Results on the TriviaQA dataset}
We experimented on the TriviaQA dataset in a similar fashion.
In particular, we used $1000$ randomly selected question-answer pairs, performed experiments with two risk tolerance levels, $\alpha=0.1$ and $\alpha=0.2$, and used $\delta=0.05$ failure probability for the confidence intervals.

Similarly to our other experiment, \Cref{fig:triviaqa} shows the trade-off between the abstention rate and test error. As a result of the fact that the LLM tends to generate shorter responses on the queries in this dataset (and the true responses are also short), log-probability scoring is competitive with our proposed scoring methods. In fact, they seem to perform quite similarly in all experiments (with the log-probability scores being somewhat better for $\alpha=0.1$ and worse for $\alpha=0.2$).
As before, we observe that there is a negligible difference between the \ac{CRC} and \ac{RCPS} methods, and that Baseline sometimes violates the risk condition with smaller calibration datasets.  

More details (with the exact numerical results) are presented in \Cref{tab:TriviaQA_test_loss_alpha_0.2,tab:TriviaQA_abs_rates_alpha_0.2,tab:TriviaQA_cal_loss_alpha_0.2,tab:TriviaQA_median_lambdas_alpha_0.2,tab:TriviaQA (temp. $=0$ response)_test_loss_alpha_0.15,tab:TriviaQA (temp. $=0$ response)_abs_rates_alpha_0.15} in \Cref{sec:additional_tables}.

Comparing the experiments for the two datasets, we can conclude that our proposed calibrated abstention methods based on match counts (or expected match counts) are preferable to the variant based on log-probability, as they perform well for both short and long answers, while the log-probability score is significantly worse for questions with long answers.

\section{Conclusions and future directions}

We proposed a conformal calibration and similarity scoring procedure which enables LLMs to abstain in a principled way.
In particular, one of our main contributions is a novel procedure to generate match scores to count the number of similar responses to a query.
When combined with conformal calibration, this scoring procedure achieves a good trade-off between abstention rate and test performance.
Importantly, in experiments over two question-answering datasets, our proposed procedure surpasses the simple baseline scoring procedure of using log-probabilities of the predictor (once more suggesting that LLMs are not well-calibrated).
Finally, we also presented a method to calibrate the match function (based on similarity measures) which is used in automatically evaluating the performance of the LLM at test time, which comes with theoretical guarantees on its accuracy and requires only a small labelled calibration set to tune the threshold.

\bibliography{refs}

\newpage
\appendix
\onecolumn

\section{\ac{CRC} in expectation}
\label{sec:crc-expectation}
\begin{lemma}
  \label{lem:crc-expectation}
  Assume that $\lambda \mapsto \ell(z, \lambda)$ is non-increasing and upper-bounded by $B$ for any example $z \in \cZ$.
  Then, for
  \begin{align*}
    \hat \lambda
    &= \inf\cbr{\lambda \in \Lambda ~:~ L_n(\lambda) + \frac{B}{n} \leq \alpha} \qquad (\alpha > 0)
  \end{align*}
  we have $(1 + \frac1n) \E R(\hat \lambda) \leq \alpha$ assuming that $D_n$ is drawn i.i.d.
\end{lemma}
\begin{proof}
For some error $\alpha > 0$ and a free parameter $\epsilon > 0$ (to be tuned later) consider solution
\begin{align*}
  \hat \lambda
  &= \inf\cbr{\lambda \in \Lambda ~:~ L(D_n; \lambda) + \epsilon \leq \alpha}~.
\end{align*}
Let $\tilde Z_i$ be an independent copy of $Z_i$ and suppose that $\hat \lambda^{+i}$ is a solution obtained by adding loss $\ell(\tilde Z_i, \lambda)/n$ to the objective, namely
\begin{align}
  \label{eq:E_stab_1}
  \hat \lambda^{+i}
  &= \inf\cbr{\lambda \in \Lambda ~:~ L(D_n; \lambda) + \frac{\ell(\tilde Z_{i}, \lambda)}{n} \leq \alpha}~.
\end{align}
Choosing $\epsilon = B/n$, we observe that $\hat \lambda^{+i} \leq \hat \lambda$ for any $i$ (the feasible set of $\hat \lambda^{+i}$ is no smaller than that of $\hat \lambda$).
Hence by the non-increasing property of the loss,
\begin{align*}
  \ell(\tilde Z_i, \hat \lambda^{+i}) \geq \ell(\tilde Z_i, \hat \lambda)
\end{align*}
while summing over losses w.r.t $(\tilde Z_1, \ldots, \tilde Z_n, \tilde Z_i)$ and taking expectation gives
\begin{align*}  
  \frac1n \sum_{j=1}^n \E[\ell(\tilde Z_j, \hat \lambda^{+i})] + \frac{\E[\ell(\tilde Z_i, \hat \lambda^{+i})]}{n}
  \geq \frac1n \sum_{j=1}^n \E[\ell(\tilde Z_j, \hat \lambda)] + \frac{\E[\ell(\tilde Z_i, \hat \lambda)]}{n} = (1 + \frac1n) \E L_n(\hat \lambda)~.
\end{align*}
On the other hand, by identicity (or exchangeability) and using the fact that $\hat \lambda^{+i}$ is a solution to \cref{eq:E_stab_1}
\begin{align*}
  \frac1n \sum_{j=1}^n \E[\ell(Z_j, \hat \lambda^{+i})] + \frac{\E[\ell(\tilde Z_i, \hat \lambda^{+i})]}{n} \leq \alpha
  \quad \Leftrightarrow \quad
  \frac1n \sum_{j=1}^n \E[\ell(\tilde Z_j, \hat \lambda^{+i})] + \frac{\E[\ell(\tilde Z_i, \hat \lambda^{+i})]}{n} \leq \alpha~.
\end{align*}
\end{proof}

\clearpage
\section{Details of the experimental setup}
\label{app:experiments}

\textbf{LLM prompting.} We use the following prompt to query the model. 

\begin{center}
\begin{tabularx}{7cm} { 
  | >{\raggedright\arraybackslash}X | }
 \hline
Answer the following question: Q: $\{\text{query}\}$ A: \\
 \hline
\end{tabularx}
\end{center}

We use the following prompt to query the model for the similarity of two responses given a question (used in computing the match count score):

\begin{center}
\begin{tabularx}{8cm} { 
  | >{\raggedright\arraybackslash}X | }
 \hline
You are given a question (Q), and two candidate answers (C and D).
Given Q, measure the similarity of C and D on a scale of 0-10.
Answer in a number between 0 and 10.
\\
\\
Q: \{question\}
\\
\\
C: \{response1\}
\\
D: \{response2\}
\\
\\
Given Q, similarity of C and D to each other is [0-10]:\\ 
\hline
\end{tabularx}
\end{center}

We also consider a more efficient implementation with a single prompt that queries similarity of one response with $K=10$ other responses (used in computing the expected match count score): 

\begin{center}
\begin{tabularx}{10cm} { 
  | >{\raggedright\arraybackslash}X | }
 \hline
You are given a question (Q), a candidate answer (C), and a list of alternative answers (A1->A10).
How many of the alternative answers are semantically the same as C, if any? Generic answers do not count.
Answer in a number between 0 and 10.
\\
\\
Q: \{question\}
\\
\\
C:  \{response\}
\\
\\
A1: \{response1\}\\
A2: \{response2\}\\
A3: \{response3\}\\
A4: \{response4\}\\
A5: \{response5\}\\
A6: \{response6\}\\
A7: \{response7\}\\
A8: \{response8\}\\
A9: \{response9\}\\
A10: \{response10\}\\
\\
The number of supportive answers for C is [0-10]:\\
 \hline
\end{tabularx}
\end{center}

\textbf{Data.} We perform experiments on a subset of a TriviaQA dataset, and on the full Temporal Sequences dataset (from BIG-bench). Below is a sample from Temporal Sequences dataset when the model generates a long answer. 

\begin{center}
\begin{tabularx}{10cm} { 
  | >{\raggedright\arraybackslash}X | }
 \hline
\textbf{Q:} Today, Jennifer went to the library. Between what times could they have gone? We know that: Jennifer woke up at 5am. William saw Jennifer working at the office from 5am to 8am. David saw Jennifer taking photos near the Leaning Tower of Pisa from 8am to 10am. Jason saw Jennifer driving to the water park from 10am to 1pm. Sarah saw Jennifer stretching at a yoga studio from 4pm to 7pm. The library was closed after 7pm. Between what times could Jennifer have gone to the library?
\\
\\
\textbf{A:} 1pm to 4pm
\\
\\
\textbf{Temp=0 response:} Jennifer could have gone to the library between 1pm and 4pm. Jennifer woke up at 5am and was seen working at the office from 5am to 8am by William. From 8am to 10am, David saw her taking photos near the Leaning Tower of Pisa. From 10am to 1pm, Jason saw her driving to the water park. This means that Jennifer was not at the library during these times. Sarah saw Jennifer stretching at a yoga studio from 4pm to 7pm. The library was closed after 7pm. This means that Jennifer could not have gone to the library after 4pm. Therefore, the only time Jennifer could have gone to the library is between 1pm and 4pm. \\
 \hline
\end{tabularx}
\end{center}

\newpage

\section{Additional experimental results}
\label{sec:additional_tables}

    \begin{table*}[!htbp]
  \caption{Dataset Temporal Sequences (temp. $=0$ response): AVERAGE TEST LOSSES. $\alpha=0.1$ (m.c. = match counts, e.m.c = expected match counts, l.p. = log-probabilities)} \begin{center}

    \resizebox{\textwidth}{!}{%
\begin{tabular}{llllllllllll}
\hline
 baseline / sample size                              & 10                & 15                & 20                & 25                & 30                & 40                & 60                & 100               & 300               & 900               & 3200              \\
\hline
 (. / e.m.c.) Baseline                & $0.098 \pm 0.019$ & $0.079 \pm 0.021$ & $0.095 \pm 0.020$ & $0.081 \pm 0.017$ & $0.085 \pm 0.014$ & $0.078 \pm 0.016$ & $0.093 \pm 0.015$ & $0.100 \pm 0.015$ & $0.084 \pm 0.005$ & $0.084 \pm 0.003$ & $0.081 \pm 0.001$ \\
 (CRC / e.m.c.)  Bound in expectation & $0.098 \pm 0.019$ & $0.040 \pm 0.021$ & $0.088 \pm 0.018$ & $0.045 \pm 0.014$ & $0.088 \pm 0.017$ & $0.075 \pm 0.015$ & $0.092 \pm 0.008$ & $0.075 \pm 0.012$ & $0.079 \pm 0.008$ & $0.082 \pm 0.004$ & $0.082 \pm 0.001$ \\
\hline
 (. / m.c.) Baseline                & $0.096 \pm 0.020$ & $0.070 \pm 0.021$ & $0.098 \pm 0.019$ & $0.059 \pm 0.022$ & $0.084 \pm 0.016$ & $0.075 \pm 0.018$ & $0.076 \pm 0.018$ & $0.097 \pm 0.019$ & $0.063 \pm 0.010$ & $0.073 \pm 0.007$ & $0.068 \pm 0.006$ \\
 (CRC / m.c.)  Bound in expectation & $0.076 \pm 0.027$ & $0.023 \pm 0.019$ & $0.074 \pm 0.021$ & $0.044 \pm 0.018$ & $0.072 \pm 0.024$ & $0.062 \pm 0.018$ & $0.073 \pm 0.017$ & $0.068 \pm 0.013$ & $0.073 \pm 0.010$ & $0.074 \pm 0.009$ & $0.065 \pm 0.004$ \\
\hline
\end{tabular}
}

\end{center}
\label{tab:Temporal_Sequences_test_loss_alpha_0.1}
\end{table*}

    \begin{table*}[!htbp]
  \caption{Dataset Temporal Sequences (temp. $=0$ response): AVERAGE TEST ABSTENTION RATES. $\alpha=0.1$ (m.c. = match counts, e.m.c. = expected match counts, l.p. = log-probabilities)} \begin{center}
\resizebox{\textwidth}{!}{%
\begin{tabular}{llllllllllll}
\hline
 baseline / sample size                              & 10                & 15                & 20                & 25                & 30                & 40                & 60                & 100               & 300               & 900               & 3200                           \\
\hline
 (. / e.m.c.) Baseline                & $0.092 \pm 0.072$ & $0.182 \pm 0.103$ & $0.106 \pm 0.080$ & $0.145 \pm 0.068$ & $0.119 \pm 0.044$ & $0.152 \pm 0.060$ & $0.098 \pm 0.055$ & $0.077 \pm 0.049$ & $0.111 \pm 0.016$ & $0.109 \pm 0.013$ & $0.117 \pm 0.007$ \\
 (CRC / e.m.c.)  Bound in expectation & $0.100 \pm 0.084$ & $0.495 \pm 0.181$ & $0.124 \pm 0.070$ & $0.313 \pm 0.081$ & $0.117 \pm 0.065$ & $0.161 \pm 0.059$ & $0.085 \pm 0.024$ & $0.149 \pm 0.040$ & $0.130 \pm 0.029$ & $0.117 \pm 0.016$ & $0.117 \pm 0.006$ \\
\hline
 (. / m.c.) Baseline                & $0.132 \pm 0.152$ & $0.154 \pm 0.076$ & $0.071 \pm 0.062$ & $0.293 \pm 0.188$ & $0.088 \pm 0.039$ & $0.119 \pm 0.059$ & $0.114 \pm 0.046$ & $0.070 \pm 0.048$ & $0.139 \pm 0.032$ & $0.097 \pm 0.013$ & $0.107 \pm 0.011$ \\
 (CRC / m.c.)  Bound in expectation & $0.214 \pm 0.159$ & $0.583 \pm 0.188$ & $0.135 \pm 0.075$ & $0.304 \pm 0.139$ & $0.164 \pm 0.088$ & $0.160 \pm 0.058$ & $0.122 \pm 0.058$ & $0.123 \pm 0.034$ & $0.103 \pm 0.024$ & $0.101 \pm 0.023$ & $0.112 \pm 0.008$ \\
\hline
\end{tabular}
}
\end{center}
\label{tab:Temporal_Sequences_abs_rates_alpha_0.1}
\end{table*}

\begin{table*}[hbt!]
  \caption{Dataset Temporal Sequences (temp. $=0$ response): AVERAGE TEST LOSSES. $\alpha=0.05$}
  \begin{center}
\resizebox{\textwidth}{!}{%
\begin{tabular}{llllllllllll}
\hline
 baseline / sample size             & 10     & 15  & 20 & 25 & 30 & 40 & 60     & 100   & 300               & 900               & 3200              \\
\hline
 (. / e.m.c.) Baseline                & $0.079 \pm 0.021$ & $0.068 \pm 0.022$ & $0.060 \pm 0.013$ & $0.060 \pm 0.020$ & $0.053 \pm 0.013$ & $0.053 \pm 0.016$ & $0.059 \pm 0.016$ & $0.037 \pm 0.007$ & $0.041 \pm 0.007$ & $0.037 \pm 0.003$ & $0.036 \pm 0.002$ \\
(CRC / e.m.c.)  Bound in expectation & $0.000 \pm 0.000$ & $0.000 \pm 0.000$ & $0.054 \pm 0.021$ & $0.026 \pm 0.011$ & $0.023 \pm 0.011$ & $0.033 \pm 0.011$ & $0.034 \pm 0.007$ & $0.035 \pm 0.009$ & $0.031 \pm 0.004$ & $0.037 \pm 0.003$ & $0.035 \pm 0.002$ \\
\hline
(. / m.c.) Baseline                & $0.029 \pm 0.021$ & $0.032 \pm 0.019$ & $0.070 \pm 0.024$ & $0.032 \pm 0.013$ & $0.038 \pm 0.022$ & $0.038 \pm 0.014$ & $0.041 \pm 0.018$ & $0.041 \pm 0.016$ & $0.025 \pm 0.009$ & $0.026 \pm 0.007$ & $0.021 \pm 0.006$ \\
 (CRC / m.c.)  Bound in expectation & $0.000 \pm 0.000$ & $0.000 \pm 0.000$ & $0.036 \pm 0.019$ & $0.011 \pm 0.009$ & $0.013 \pm 0.008$ & $0.027 \pm 0.010$ & $0.036 \pm 0.014$ & $0.027 \pm 0.013$ & $0.023 \pm 0.007$ & $0.029 \pm 0.007$ & $0.015 \pm 0.000$ \\
\hline
\end{tabular}
}
\end{center}
\label{tab:Temporal Sequences (temp. $=0$ response)_test_loss_alpha_0.05}
\end{table*}

\begin{table*}[hbt!]
  \caption{Dataset Temporal Sequences (temp. $=0$ response): AVERAGE TEST ABSTENTION RATES. $\alpha=0.05$}\begin{center}
\resizebox{\textwidth}{!}{%
\begin{tabular}{llllllllllll}
\hline
 baseline / sample size             & 10     & 15  & 20 & 25 & 30 & 40 & 60     & 100   & 300               & 900               & 3200              \\
\hline
 (. | e.m.c.) Baseline                & $0.199 \pm 0.120$ & $0.227 \pm 0.100$ & $0.233 \pm 0.076$ & $0.277 \pm 0.125$ & $0.278 \pm 0.085$ & $0.277 \pm 0.085$ & $0.249 \pm 0.088$ & $0.337 \pm 0.056$ & $0.301 \pm 0.043$ & $0.318 \pm 0.019$ & $0.321 \pm 0.012$ \\
(CRC | e.m.c.)  Bound in expectation & $1.000 \pm 0.000$ & $1.000 \pm 0.000$ & $0.356 \pm 0.170$ & $0.536 \pm 0.129$ & $0.609 \pm 0.155$ & $0.419 \pm 0.106$ & $0.359 \pm 0.053$ & $0.376 \pm 0.067$ & $0.368 \pm 0.030$ & $0.322 \pm 0.020$ & $0.331 \pm 0.011$ \\
\hline
 (. | m.c.) Baseline                & $0.509 \pm 0.183$ & $0.475 \pm 0.190$ & $0.165 \pm 0.079$ & $0.399 \pm 0.167$ & $0.407 \pm 0.172$ & $0.363 \pm 0.171$ & $0.268 \pm 0.077$ & $0.262 \pm 0.077$ & $0.337 \pm 0.055$ & $0.327 \pm 0.050$ & $0.366 \pm 0.041$ \\
 (CRC | m.c.)  Bound in expectation & $1.000 \pm 0.000$ & $1.000 \pm 0.000$ & $0.396 \pm 0.170$ & $0.723 \pm 0.179$ & $0.663 \pm 0.179$ & $0.377 \pm 0.121$ & $0.330 \pm 0.130$ & $0.391 \pm 0.121$ & $0.346 \pm 0.047$ & $0.307 \pm 0.051$ & $0.405 \pm 0.000$ \\
\hline
\end{tabular}
}
\end{center}
\label{tab:Temporal Sequences (temp. $=0$ response)_abs_rates_alpha_0.05}
\end{table*}


\begin{table*}[hbt!]
  \caption{Dataset TriviaQA (temp. $=0$ response): AVERAGE TEST LOSSES. $\alpha=0.1$}\begin{center}
\resizebox{\textwidth}{!}{%
\begin{tabular}{lllllllllll}
\hline
 baseline / sample size             & 10 & 15 & 20 & 25 & 30 & 40 & 60 & 100 & 300 & 800    \\
 \hline
 (. / e.m.c.) Baseline                & $0.128 \pm 0.038$ & $0.092 \pm 0.031$ & $0.125 \pm 0.028$ & $0.086 \pm 0.023$ & $0.097 \pm 0.022$ & $0.090 \pm 0.015$ & $0.071 \pm 0.008$ & $0.081 \pm 0.011$ & $0.076 \pm 0.005$ & $0.072 \pm 0.002$ \\
 (CRC / e.m.c.)  Bound in expectation & $0.089 \pm 0.031$ & $0.066 \pm 0.031$ & $0.064 \pm 0.021$ & $0.049 \pm 0.015$ & $0.060 \pm 0.022$ & $0.083 \pm 0.020$ & $0.061 \pm 0.011$ & $0.073 \pm 0.008$ & $0.073 \pm 0.007$ & $0.071 \pm 0.003$ \\
\hline
 (. / m.c.) Baseline                & $0.067 \pm 0.044$ & $0.032 \pm 0.034$ & $0.064 \pm 0.043$ & $0.048 \pm 0.039$ & $0.020 \pm 0.032$ & $0.026 \pm 0.027$ & $0.012 \pm 0.019$ & $0.012 \pm 0.019$ & $0.000 \pm 0.000$ & $0.000 \pm 0.000$ \\
 (CRC / m.c.)  Bound in expectation & $0.076 \pm 0.041$ & $0.073 \pm 0.040$ & $0.029 \pm 0.031$ & $0.012 \pm 0.019$ & $0.015 \pm 0.024$ & $0.000 \pm 0.000$ & $0.012 \pm 0.019$ & $0.012 \pm 0.019$ & $0.000 \pm 0.000$ & $0.000 \pm 0.000$ \\
\hline
  \hline
 (. / l.p.) Baseline                & $0.116 \pm 0.031$ & $0.094 \pm 0.029$ & $0.140 \pm 0.025$ & $0.106 \pm 0.029$ & $0.107 \pm 0.035$ & $0.100 \pm 0.025$ & $0.091 \pm 0.014$ & $0.073 \pm 0.015$ & $0.071 \pm 0.010$ & $0.066 \pm 0.002$ \\
\hline
\end{tabular}
}
\end{center}
\label{tab:TriviaQA (temp. $=0$ response)_test_loss_alpha_0.15}
\end{table*}

\begin{table*}[hbt!]
  \caption{Dataset TriviaQA (temp. $=0$ response): AVERAGE TEST ABSTENTION RATES. $\alpha=0.1$}\begin{center}
\resizebox{\textwidth}{!}{%
\begin{tabular}{lllllllllll}
\hline
 baseline / sample size            & 10 & 15 & 20 & 25 & 30 & 40 & 60 & 100 & 300 & 800 \\
\hline
(. / e.m.c.) Baseline                & $0.347 \pm 0.198$ & $0.524 \pm 0.167$ & $0.341 \pm 0.153$ & $0.575 \pm 0.136$ & $0.497 \pm 0.135$ & $0.501 \pm 0.114$ & $0.636 \pm 0.064$ & $0.568 \pm 0.093$ & $0.607 \pm 0.047$ & $0.643 \pm 0.025$ \\
 (CRC / e.m.c.)  Bound in expectation & $0.511 \pm 0.185$ & $0.679 \pm 0.170$ & $0.673 \pm 0.146$ & $0.770 \pm 0.086$ & $0.677 \pm 0.137$ & $0.558 \pm 0.136$ & $0.727 \pm 0.084$ & $0.648 \pm 0.068$ & $0.631 \pm 0.060$ & $0.653 \pm 0.029$ \\
\hline
 (. / m.c.) Baseline                & $0.631 \pm 0.236$ & $0.813 \pm 0.195$ & $0.640 \pm 0.231$ & $0.727 \pm 0.218$ & $0.900 \pm 0.156$ & $0.831 \pm 0.176$ & $0.919 \pm 0.127$ & $0.919 \pm 0.127$ & $1.000 \pm 0.000$ & $1.000 \pm 0.000$ \\
 (CRC / m.c.)  Bound in expectation & $0.556 \pm 0.233$ & $0.565 \pm 0.228$ & $0.825 \pm 0.183$ & $0.919 \pm 0.127$ & $0.908 \pm 0.143$ & $1.000 \pm 0.000$ & $0.919 \pm 0.127$ & $0.919 \pm 0.127$ & $1.000 \pm 0.000$ & $1.000 \pm 0.000$ \\
\hline
 (. / l.p.) Baseline                & $0.381 \pm 0.111$ & $0.499 \pm 0.114$ & $0.310 \pm 0.102$ & $0.410 \pm 0.096$ & $0.484 \pm 0.155$ & $0.468 \pm 0.062$ & $0.493 \pm 0.023$ & $0.528 \pm 0.029$ & $0.523 \pm 0.020$ & $0.537 \pm 0.006$ \\
\hline
\end{tabular}
}
\end{center}
\label{tab:TriviaQA (temp. $=0$ response)_abs_rates_alpha_0.15}
\end{table*}


    \begin{table*}[!htbp]
  \caption{Dataset TriviaQA (temp. $=0$ response): AVERAGE TEST LOSSES. $\alpha=0.2$ (m.c. = match counts, e.m.c. = expected match counts, l.p. = log-probabilities)} \begin{center}

    \resizebox{\textwidth}{!}{%
\begin{tabular}{lllllllllll}
\hline
 baseline / sample size                              & 10                & 15                & 20                & 25                & 30                & 40                & 60                & 100               & 300               & 800               \\
\hline
 (. / e.m.c.) Baseline                & $0.157 \pm 0.031$ & $0.165 \pm 0.031$ & $0.191 \pm 0.011$ & $0.178 \pm 0.017$ & $0.163 \pm 0.022$ & $0.169 \pm 0.016$ & $0.155 \pm 0.012$ & $0.154 \pm 0.016$ & $0.150 \pm 0.012$ & $0.145 \pm 0.004$ \\
 (CRC / e.m.c.)  Bound in expectation & $0.148 \pm 0.031$ & $0.152 \pm 0.022$ & $0.116 \pm 0.022$ & $0.127 \pm 0.026$ & $0.134 \pm 0.030$ & $0.137 \pm 0.020$ & $0.141 \pm 0.017$ & $0.152 \pm 0.011$ & $0.140 \pm 0.004$ & $0.142 \pm 0.005$ \\
\hline
 (. / m.c.) Baseline                & $0.153 \pm 0.033$ & $0.175 \pm 0.016$ & $0.167 \pm 0.033$ & $0.172 \pm 0.032$ & $0.131 \pm 0.046$ & $0.156 \pm 0.031$ & $0.132 \pm 0.027$ & $0.133 \pm 0.027$ & $0.143 \pm 0.010$ & $0.134 \pm 0.004$ \\
 (CRC / m.c.)  Bound in expectation & $0.125 \pm 0.044$ & $0.138 \pm 0.038$ & $0.072 \pm 0.039$ & $0.097 \pm 0.044$ & $0.075 \pm 0.048$ & $0.129 \pm 0.036$ & $0.097 \pm 0.035$ & $0.137 \pm 0.026$ & $0.135 \pm 0.005$ & $0.134 \pm 0.006$ \\
\hline
 (. / l.p.) Baseline                & $0.175 \pm 0.019$ & $0.165 \pm 0.023$ & $0.184 \pm 0.016$ & $0.177 \pm 0.022$ & $0.164 \pm 0.020$ & $0.169 \pm 0.014$ & $0.162 \pm 0.014$ & $0.150 \pm 0.012$ & $0.151 \pm 0.006$ & $0.144 \pm 0.000$ \\
\hline
\end{tabular}
}

\end{center}
\label{tab:TriviaQA_test_loss_alpha_0.2}
\end{table*}

    \begin{table*}[!htbp]
  \caption{Dataset TriviaQA (temp. $=0$ response): AVERAGE TEST ABSTENTION RATES. $\alpha=0.2$ (m.c. = match counts, e.m.c. = expected match counts, l.p. = log-probabilities)} \begin{center}
\resizebox{\textwidth}{!}{%
\begin{tabular}{lllllllllll}
\hline
 baseline / sample size                              & 10                & 15                & 20                & 25                & 30                & 40                & 60                & 100               & 300               & 800               \\
\hline
 (. / e.m.c.) Baseline                & $0.202 \pm 0.151$ & $0.183 \pm 0.164$ & $0.029 \pm 0.021$ & $0.053 \pm 0.035$ & $0.130 \pm 0.080$ & $0.070 \pm 0.031$ & $0.096 \pm 0.026$ & $0.105 \pm 0.033$ & $0.106 \pm 0.020$ & $0.110 \pm 0.010$ \\
 (CRC / e.m.c.)  Bound in expectation & $0.240 \pm 0.157$ & $0.163 \pm 0.105$ & $0.339 \pm 0.133$ & $0.292 \pm 0.150$ & $0.311 \pm 0.149$ & $0.198 \pm 0.092$ & $0.165 \pm 0.065$ & $0.102 \pm 0.024$ & $0.123 \pm 0.014$ & $0.115 \pm 0.012$ \\
\hline
 (. / m.c.) Baseline                & $0.169 \pm 0.150$ & $0.061 \pm 0.036$ & $0.133 \pm 0.154$ & $0.121 \pm 0.153$ & $0.322 \pm 0.232$ & $0.153 \pm 0.150$ & $0.208 \pm 0.141$ & $0.207 \pm 0.142$ & $0.120 \pm 0.026$ & $0.130 \pm 0.016$ \\
 (CRC / m.c.)  Bound in expectation & $0.334 \pm 0.228$ & $0.245 \pm 0.198$ & $0.562 \pm 0.229$ & $0.460 \pm 0.232$ & $0.610 \pm 0.249$ & $0.269 \pm 0.193$ & $0.399 \pm 0.207$ & $0.185 \pm 0.143$ & $0.133 \pm 0.020$ & $0.140 \pm 0.021$ \\
\hline
 (. / l.p.) Baseline                & $0.172 \pm 0.108$ & $0.211 \pm 0.111$ & $0.120 \pm 0.094$ & $0.133 \pm 0.099$ & $0.214 \pm 0.097$ & $0.195 \pm 0.086$ & $0.239 \pm 0.082$ & $0.324 \pm 0.075$ & $0.322 \pm 0.044$ & $0.372 \pm 0.008$ \\
\hline
\end{tabular}
}
\end{center}
\label{tab:TriviaQA_abs_rates_alpha_0.2}
\end{table*}

    \begin{table}[!htbp]
\caption{Dataset Temporal Sequences (temp. $=0$ response): UPPER CONFIDENCE BOUNDS ON CALIBRATION SET. $\alpha=0.1$} \begin{center}
\resizebox{\textwidth}{!}{%
\begin{tabular}{llllllllllll}
\hline
 baseline / sample size                              & 10                 & 15                & 20                & 25                & 30                & 40                & 60                & 100               & 300               & 900               & 3200              \\
\hline
 (. / e.m.c.) Baseline                & $0.090 \pm 0.016$  & $0.060 \pm 0.010$ & $0.080 \pm 0.013$ & $0.080 \pm 0.000$ & $0.093 \pm 0.007$ & $0.092 \pm 0.008$ & $0.093 \pm 0.006$ & $0.095 \pm 0.004$ & $0.097 \pm 0.001$ & $0.096 \pm 0.001$ & $0.097 \pm 0.001$ \\
 (CRC / e.m.c.)  Bound in expectation & $0.091 \pm 0.000$  & $0.062 \pm 0.000$ & $0.095 \pm 0.000$ & $0.073 \pm 0.006$ & $0.094 \pm 0.005$ & $0.093 \pm 0.005$ & $0.090 \pm 0.006$ & $0.095 \pm 0.003$ & $0.097 \pm 0.002$ & $0.097 \pm 0.002$ & $0.095 \pm 0.002$ \\
\hline
 (. / m.c.) Baseline                & $0.070 \pm 0.024$ & $0.053 \pm 0.014$ & $0.080 \pm 0.013$ & $0.048 \pm 0.018$ & $0.093 \pm 0.007$ & $0.090 \pm 0.006$ & $0.085 \pm 0.006$ & $0.086 \pm 0.006$ & $0.083 \pm 0.006$ & $0.089 \pm 0.003$ & $0.088 \pm 0.003$ \\
 (CRC / m.c.)  Bound in expectation & $0.091 \pm 0.000$ & $0.062 \pm 0.000$ & $0.081 \pm 0.011$ & $0.069 \pm 0.008$ & $0.084 \pm 0.008$ & $0.080 \pm 0.008$ & $0.089 \pm 0.007$ & $0.088 \pm 0.008$ & $0.093 \pm 0.003$ & $0.090 \pm 0.005$ & $0.086 \pm 0.002$ \\
\hline
\end{tabular}
}
\end{center}
\label{tab:Temporal Sequences_cal_loss_alpha_0.1}
\end{table}

    \begin{table*}[hbt!]
  \caption{Dataset Temporal Sequences (temp. $=0$ response): MEDIAN $\lambda$ values. $\alpha=0.1$} \begin{center}
\resizebox{\textwidth}{!}{%
\begin{tabular}{llllllllllll}
\hline
 baseline / sample size                              & 10                             & 15                            & 20                             & 25                            & 30                             & 40                             & 60                             & 100                            & 300                            & 900                            & 3200                          \\
\hline
 (. / e.m.c.) Baseline                & $1.253 \pm 1.080$ & $3.272 \pm 1.041$ & $1.183 \pm 1.131$ & $3.550 \pm 0.907$ & $3.550 \pm 0.815$  & $3.689 \pm 0.899$ & $2.993 \pm 0.919$ & $1.253 \pm 0.997$ & $3.481 \pm 0.160$  & $3.341 \pm 0.118$ & $3.481 \pm 0.056$ \\
 (CRC / e.m.c.)  Bound in expectation & $1.113 \pm 1.103$ & $5.639 \pm 1.080$ & $3.272 \pm 0.941$  & $4.873 \pm 0.520$ & $3.202 \pm 0.979$   & $3.550 \pm 0.733$  & $3.132 \pm 0.540$  & $3.898 \pm 0.649$ & $3.550 \pm 0.241$  & $3.481 \pm 0.146$  & $3.481 \pm 0.046$ \\
\hline
 (. / m.c.) Baseline                & $4.0 \pm 1.908$  & $8.0 \pm 1.710$  & $3.5 \pm 1.803$  & $8.5 \pm 1.775$  & $7.0 \pm 1.468$  & $8.0 \pm 1.547$  & $8.0 \pm 1.799$  & $3.5 \pm 1.849$  & $8.0 \pm 0.432$  & $8.0 \pm 0.255$  & $8.0 \pm 0.208$ \\
 (CRC / m.c.)  Bound in expectation & $4.5 \pm 2.268$  & $10.0 \pm 1.493$ & $8.0 \pm 1.757$  & $9.0 \pm 0.864$  & $7.5 \pm 1.919$  & $8.5 \pm 1.381$  & $8.0 \pm 1.295$  & $8.0 \pm 1.168$  & $8.0 \pm 0.416$  & $7.5 \pm 0.345$  & $8.0 \pm 0.156$ \\
\hline
\end{tabular}
}
\end{center}
\label{tab:Temporal Sequences_median_lambdas_alpha_0.1}
\end{table*}


    \begin{table*}[hbt!]
\caption{Dataset TriviaQA (temp. $=0$ response): UPPER CONFIDENCE BOUNDS ON CALIBRATION SET. $\alpha=0.2$} \begin{center}
\resizebox{\textwidth}{!}{%
\begin{tabular}{lllllllllll}
\hline
 baseline / sample size                              & 10                 & 15                & 20                & 25                & 30                & 40                & 60                & 100               & 300               & 800               \\
\hline
 (. / e.m.c.) Baseline                & $0.180 \pm 0.021$  & $0.187 \pm 0.014$ & $0.180 \pm 0.017$ & $0.180 \pm 0.019$ & $0.177 \pm 0.017$ & $0.192 \pm 0.008$ & $0.193 \pm 0.006$ & $0.196 \pm 0.003$ & $0.196 \pm 0.001$ & $0.197 \pm 0.002$ \\
 (CRC / e.m.c.)  Bound in expectation & $0.173 \pm 0.014$  & $0.169 \pm 0.021$ & $0.176 \pm 0.011$ & $0.188 \pm 0.006$ & $0.187 \pm 0.007$ & $0.190 \pm 0.005$ & $0.190 \pm 0.006$ & $0.196 \pm 0.002$ & $0.195 \pm 0.002$ & $0.197 \pm 0.002$ \\
\hline
 (. / m.c.) Baseline                & $0.150 \pm 0.035$ & $0.173 \pm 0.017$ & $0.160 \pm 0.032$ & $0.156 \pm 0.033$ & $0.123 \pm 0.045$ & $0.170 \pm 0.031$ & $0.175 \pm 0.031$ & $0.170 \pm 0.031$ & $0.186 \pm 0.007$ & $0.185 \pm 0.007$ \\
 (CRC / m.c.)  Bound in expectation & $0.145 \pm 0.023$ & $0.144 \pm 0.029$ & $0.110 \pm 0.033$ & $0.127 \pm 0.038$ & $0.097 \pm 0.041$ & $0.156 \pm 0.035$ & $0.133 \pm 0.040$ & $0.172 \pm 0.029$ & $0.183 \pm 0.005$ & $0.183 \pm 0.007$ \\
\hline
 (. / l.p.) Baseline                & $0.180 \pm 0.021$  & $0.187 \pm 0.014$  & $0.180 \pm 0.017$ & $0.180 \pm 0.019$ & $0.183 \pm 0.016$ & $0.192 \pm 0.008$ & $0.193 \pm 0.008$ & $0.196 \pm 0.003$ & $0.198 \pm 0.001$ & $0.193 \pm 0.002$ \\
\hline
\end{tabular}
}
\end{center}
\label{tab:TriviaQA_cal_loss_alpha_0.2}
\end{table*}

\begin{table*}[hbt!]
  \caption{Dataset TriviaQA (temp. $=0$ response): MEDIAN $\lambda$ values. $\alpha=0.2$} \begin{center}
\resizebox{\textwidth}{!}{%
\begin{tabular}{lllllllllll}
\hline
 baseline / sample size                              & 10                             & 15                             & 20                            & 25                             & 30                             & 40                             & 60                             & 100                            & 300                            & 800                            \\
\hline
 (. / e.m.c.) Baseline                & $1.322 \pm 1.395$ & $1.879 \pm 1.286$ & $0.0 \pm 0.845$               & $0.974 \pm 1.101$ & $2.715 \pm 1.210$ & $2.993 \pm 1.011$ & $3.829 \pm 0.725$  & $4.037 \pm 0.875$  & $4.177 \pm 0.437$  & $4.107 \pm 0.237$ \\
 (CRC / e.m.c.)  Bound in expectation & $2.854 \pm 1.194$ & $3.132 \pm 1.034$  & $5.430 \pm 0.938$ & $5.430 \pm 1.194$  & $5.082 \pm 1.314$  & $4.594 \pm 0.890$  & $4.734 \pm 0.904$  & $3.620 \pm 0.489$  & $4.107 \pm 0.288$ & $4.246 \pm 0.279$  \\
\hline
 (. / m.c.) Baseline                & $5.0 \pm 2.352$  & $5.5 \pm 1.890$  & $1.0 \pm 2.112$  & $2.0 \pm 1.814$  & $5.0 \pm 2.112$  & $5.0 \pm 1.949$  & $9.5 \pm 1.296$  & $9.5 \pm 1.586$  & $9.0 \pm 0.666$  & $9.0 \pm 0.280$  \\
 (CRC / m.c.)  Bound in expectation & $5.5 \pm 1.907$  & $7.0 \pm 1.820$  & $10.5 \pm 1.511$ & $10.0 \pm 2.047$ & $11.0 \pm 1.980$ & $8.0 \pm 1.615$  & $10.0 \pm 1.468$ & $8.5 \pm 0.999$  & $9.0 \pm 0.364$  & $9.0 \pm 0.389$  \\
\hline
\end{tabular}
}
\end{center}
\label{tab:TriviaQA_median_lambdas_alpha_0.2}
\end{table*}

\end{document}